\documentclass{article}

\usepackage{ijcai19}

\usepackage{multirow,multicol}
\usepackage{amsmath}
\usepackage{tabularx}
\usepackage{makecell}
\usepackage{algorithm}
\usepackage{subfig}
\usepackage[noend]{algpseudocode}
\usepackage{url} 

\usepackage{times}
\usepackage{helvet}
\usepackage{courier}
\usepackage{bm}
\usepackage{amsmath}
\usepackage{mdwlist}
\usepackage[noend]{algpseudocode}
\usepackage{relsize}
\usepackage{amsfonts}
\usepackage{graphicx}
\usepackage{caption}
\usepackage{hhline}
\usepackage{color}
\usepackage{colortbl}
\usepackage{booktabs}
\usepackage{textcomp}
\usepackage{amsthm}

\newcommand{\rbm}[1]{\bm{\mathrm{#1}}}

\makeatletter
\def\BState{\State\hskip-\ALG@thistlm}
\makeatother

\DeclareMathOperator*{\argmin}{arg\,min}

\makeatletter

\def\BState{\State\hskip-\ALG@thistlm}
\makeatother

\newenvironment{theorem}[2][Theorem]{\begin{trivlist}
\item[\hskip \labelsep {\bfseries #1}\hskip \labelsep {\bfseries #2.}]}{\end{trivlist}}

\title{Address Instance-level Label Prediction in Multiple Instance Learning}

\author{
Minlong Peng
\and
Qi Zhang
\affiliations
Fudan University\\
\emails
\{mlpeng16, qz\}@fudan.edu.cn
}

\begin{document}

\maketitle

\begin{abstract}
\textit{Multiple Instance Learning} (MIL) is concerned with learning from bags of instances, where only bag labels are given and instance labels are unknown. Existent approaches in this field were mainly designed for the bag-level label prediction (predict labels for bags) but not the instance-level (predict labels for instances), with the task loss being only defined at the bag level. This restricts their application in many tasks, where the instance-level labels are more interested. In this paper, we propose a novel algorithm, whose lose is specifically defined at the instance level, to address instance-level label prediction in MIL. We prove that the loss of this algorithm can be unbiasedly and consistently estimated without using instance labels, under the i.i.d assumption.  Empirical study validates the above statements and shows that the proposed algorithm can achieve superior instance-level and comparative bag-level performance, compared to state-of-the-art MIL methods. In addition, it shows that the proposed method can achieve similar results as the fully supervised model (trained with instance labels) for label prediction at the instance level.

\end{abstract}

\section{Introduction}
In a multiple instance learning (MIL) task, we are to learn a classifier based on a set of bags, where each bag contains multiple instances. And in the setting of MIL, We know the label of each bag but do not know the label of each instance.
This applies to many real-world tasks, such as medical imaging \cite{quellec2017multiple} (e.g., computational pathology, mammography or CT lung screening), drug discovery (pharmacy) \cite{dietterich1997solving}, classification of text documents \cite{ruiz2015multiple}, speaker identification (signal processing) \cite{mandel2008multiple}, and so on. Therefore,  MIL is an important topic in the machine learning community and many methods have been published in the last few years

Existent MIL approaches can be roughly categorized into the instance-level paradigm and the bag-level paradigm. The instance-level methods \cite{zhang2002dd,ray2005supervised,angelidis2018multiple} treat instances of a bag differently and first predict instance labels as interim results. Then, they infer the bag label based on the predicted labels of instances constituting the bag. 
The bag-level approaches treat the bag as a whole and directly obtain the bag representation, implicitly by defining distance between bags \cite{wang2000solving}, bag kernels \cite{gartner2002multi} and bag dissimilarities \cite{cheplygina2015multiple}, or explicitly by pooling (may with attention) instance presentations \cite{ilse2018attention}. Then, they infer the bag label based on the resultant bag representation.

Because there are only bag labels, most of these methods were designed to accurately predict bag labels and their loss functions were only defined at the bag level. Even for methods of the instance-level paradigm, the instance-level label prediction is only an interim step for the final bag-level label prediction and no loss is defined on this goal. 
Therefore, these methods generally perform insufficiently for label prediction at the instance level \cite{kandemir2015computer,cheplygina2015multiple}. This restricts their applications in many tasks, where the instance-level label prediction is more interested, such as image segmentation and fine-grained sentiment classification. 

In this work, we propose a novel MIL algorithm, whose loss function is specifically defined on the instance-level label prediction, to address this problem. The core idea of this algorithm is to unbiasedly estimate the instance-level label prediction loss without using instance labels. We show that this can be achieved if we know the ratio of negative instance under the i.i.d assumption in Theorem \textbf{1}. Further theoretical analysis on the consistency of this algorithm in Theorem \textbf{2} shows that, when using a bounded Bayes consistent loss function, it can achieve similar results as the fully supervised method (trained with instance labels) for label prediction at the instance level. Our experimental study validate the effectiveness of this algorithm at both the bag level and instance level. 
Contributions of this work can be summarized as follows:
\begin{itemize}
\item We propose a novel MIL algorithm, whose loss function is specifically defined on the instance-level label prediction, to address instance-level label prediction in MIL. 
\item We provide a theoretical analysis on the proposed algorithm and prove that, by using a bounded Bayes consistent loss function, the instance-level label prediction loss can be unbiasedly and consistently estimated without knowing instance labels. 
\item Experiment results on both image and text datasets verify the effectiveness of this algorithm for label prediction at both the bag- and instance- levels.
\end{itemize}

\section{Related Work}
Since Dietterich et al. \citeyear{dietterich1997solving} introduced multiple instance learning for drug activity prediction, researchers have proposed a large number of algorithms for the MIL tasks. According to how the information existent in the multiple instance (MI) data is exploited, these algorithms can be categorized into the instance-level paradigm and the bag-level paradigm. 

In the instance-level paradigm, Diverse Density \cite{maron1998framework} is perhaps the best known framework for MIL. Formally, for a subject $t$, denote $DD(t)=\frac{1}{Z}\prod_i \text{P}(t|B_{i}^+) \prod_i \text{P}(t|B_i^-)$, where $B_i^+$ and $B_i^-$ represent the $i^{th}$ positive and negative bag respectively, and $\text{P}(t|B_{i}^+)$ and $Pr(t|B_{i}^-)$ are defined by:
\begin{align*}
    & \text{P}(t|B_{i}^+) = 1- \prod_j (1-\text{P}(t|B_{ij}^+)) \\
    & \text{P}(t|B_{i}^-) = \prod_j (1-\text{P}(t|B_{ij}^-)),
\end{align*}
where $B^+_{ij}$ and $B^-_{ij}$ is the $j^{th}$ instance in bag $B^+_{i}$ and $B^-_{i}$. The object of this algorithm is to seek the $t$ that maximizes $DD(t)$. Following this framework, many methods have been proposed \cite{ray2005supervised,jia2008instance}. The main difference of these methods lies in the definition of $\text{Pr}(t|B_{i}^+)$ and $\text{Pr}(t|B_{i}^-)$. For example, the Multiple Instance Logistic Regression (MILR) \cite{ray2005supervised} method defines:
\begin{align*}
&\text{P}(t|B_{i}^+) = \text{softmax}_{\alpha} (\text{P}(t|B_{i1}^+), \cdots, \text{P}(t|B_{im}^+) )\\
&\text{P}(t|B_{i}^-) = 1-\text{softmax}_{\alpha} (\text{P}(t|B_{i1}^-), \cdots, \text{P}(t|B_{im}^-))
\end{align*}
where 
\[\text{softmax}_{\alpha}(x_1, \cdots, x_m) = {(\sum_i^m x_ie^{\alpha x_i})}/{(\sum_i^m e^{\alpha x_i}}).\]
More recent neural-network-based methods replaces the $\text{softmax}_{\alpha}$ function with the log-sum-exp pooling function \cite{pinheiro2015image,wang2018revisiting}, the max pooling function \cite{wu2015deep}, a weighted function \cite{pappas2017explicit}, and a gated attention function \cite{angelidis2018multiple}. 

Kotzias et al., \citeyear{kotzias2015group} in addition proposed a constrain on the instance-level label prediction based on similarity between instances. They encouraged label predictions to be close for similar instances.
In the line of ultilizing instance similarity, some methods proposed to directly recognize key (positive) instances \cite{zhang2002dd,liu2012key}. The biggest difficulty of these similarity-based methods is that it is hard to design an appropriate distance measure between instances, especially for non-vectorized data like image and text.

In the bag-level paradigm, each bag is treated as a whole. Methods of this paradigm commonly define of a distance function $D(X, Y)$ that provides a way of comparing any two non-vectorial bags $X$ and $Y$. Once this distance function has been defined, it can be used in any standard distance-based classifier such as \textit{Citation}-KNN \cite{wang2000solving}, and SVM \cite{andrews2003support,kwok2007marginalized,doran2014theoretical}.
A representative framework in this paradigm is the MI-kernel \cite{gartner2002multi}, which defines a kernel as a sum of the instance kernels and represents each bag by the minimum and maximum feature values of its instances.
Along the same line, Zhou et al., \citeyear{zhou2009multi} assumed that instances within a bag are not identically independent. Based on this assumption, they proposed another kernel function, which not only used the similarity between pairwise $(x, y)$ where $x \in X$ and $y \in Y$, but also used the similarity between the neighborhood of $x$ in $X$ and the neighborhood of $y$ in $Y$. 
Some other methods directly obtain the bag representation. For example,  Wei et al., \citeyear{wei2017scalable} used a Fisher transform, and ILse et al., \citeyear{ilse2018attention} applied an attention-based pooling algorithm to to obtain the bag representation. The bag representation was further processed by a classifier to the perform the bag-level label prediction.

\section{Preliminaries and Problem Statement}

\subsection{Multiple Instance Learning}
Let $\rbm{X} \in \mathcal{X}$ and $\rm{Y}_{\rbm{X}} \in \mathcal{Y}_{\mathcal{X}}$ be the instance-level input and output random variables, where $\mathcal{X} \subset R^d$ and $\mathcal{Y}_{\mathcal{X}} = \{0, 1\}$ denote the space of $\rbm{X}$ and $\rm{Y}_{\rbm{X}}$, respectively. Let $\rbm{B} \in \mathcal{B}$ and $\rm{Y}_{\rbm{B}} \in \mathcal{Y}_{\mathcal{B}}$ be the bag-level input and output random variables, where $\mathcal{B} = \{\mathcal{X}\}$ and $\mathcal{Y}_{\mathcal{B}} = \{0, 1\}$ denote the space of $\rbm{B}$ and $\rm{Y}_{\rbm{B}}$, respectively. In a typical multiple instance learning problem, we are given a set of bags and the bag labels, $S = \{(\bm{b}_1, y_{\bm{b}_1}), (\bm{b}_2, y_{\bm{b}_2}), \cdots, (\bm{b}_N, y_{\bm{b}_N})\} \subseteq \mathcal{B} \times \mathcal{Y}_{\mathcal{B}}$, drawn iid according to an unknown distribution, $\mathcal{D}$, over $\mathcal{B} \times \mathcal{Y}_{\mathcal{B}}$. And label of each instance $y_{\bm{x}}$ is unknown. Let $S^+$ and $S^-$ denote the positive ($\rbm{Y}_{\rbm{B}}=1$) and negative ($\rbm{Y}_{\rbm{B}}=0$) bag set, respectively. The standard MIL setting states that:
\begin{itemize}
    \item If bag $\bm{b}$ is negative, then all instances in $\bm{b}$ is negative, i.e.,  $y_{\bm{x}} = 0, \forall \bm{x} \in \bm{b}, \forall \bm{b} \in S^-$.
    \item If bag $\bm{b}$ is positive, then at least one instance in $\bm{b}$ is positive, i.e., $\exists \bm{x} \in \bm{b}, y_{\bm{x}} = 1, \forall \bm{b} \in S^+$.
\end{itemize}

\subsection{Instance-level Risk Minimization in MIL}
In this work, we focus on the instance-level label prediction and solve it from the view of risk minimization.
Let $f:\mathcal{X} \rightarrow \mathcal{C},\; \mathcal{C} \subseteq R$ denotes a classifier, and $h(\bm{x}) = \text{pred} \circ f(\bm{x}): \mathcal{X} \rightarrow \mathcal{Y}_{\mathcal{X}}$ denotes the predicted label of $f(\bm{x})$. A loss function is a map $L: \mathcal{C} \times \mathcal{Y} \rightarrow R^+$. Given any loss function, $L$, and
a classifier, $f$, we define the instance-level $L$-risk of $f$ by:
\begin{equation}
R_L(f) = \mathbb{E}_{\rbm{X}, \rm{Y}_{\rbm{X}}} L(f(\bm{x}), y_{\bm{x}})
\end{equation}
where, as a notation throughout this paper, the $\mathbb{E}$ denotes expectation and
its subscript indicates the random variables with respect to which the expectation is taken. When $L$ happens to be the $0 - 1$ loss, $R_L$ would be the typical Bayes risk:
\begin{equation}
	R_{0-1}(f) = \mathbb{E}_{\rbm{X}, \rm{Y}_{\rbm{X}}} \mathbb{I}(\text{pred} \circ f(\bm{x}) \neq y_{\bm{x}}),
\end{equation}
where $\mathbb{I}$ is the indicator function.
In the risk minimization framework, the objective of this work is to learn a classifier, $f^{*}$, that minimizes $R_L$ under the bag constrains of MIL:
\begin{align} \label{eq:3}
\begin{split}
	f^* &= \argmin_f R_L(f). \\
	&s.t. \quad y_{\bm{x}} = 0, \forall \bm{x} \in \bm{b}, \forall \bm{b} \in S^- \\
 	&\quad \quad \; \exists \; \bm{x} \in \bm{b}, y_{\bm{x}} = 1, \forall \bm{b} \in S^+
\end{split}
\end{align}
The biggest challenge of the above optimization problem is that we do not know the labels of instances. Therefore, it is impossible to estimate $R_L(f)$ with the empirical risk:
\begin{equation} \label{eq:sup_risk_estimate}
\hat{R}^s_L(f) = \frac{1}{\sum_{\bm{b} \in S} |b|} \sum_{\bm{b} \in S} \sum_{\bm{x} \in \bm{b}} L(f(\bm{x}), y_{\bm{x}}),
\end{equation}
where the superscript $s$ of $\hat{R}^s_L$ means the risk is estimated in the supervised setting, where instance labels are available.
In addition, because the $0-1$ loss is non-differentiable, it is necessary to replace it with an appropriate loss function $L$. A common requirement for $L$ is that it should be differentiable. Another critical requirement is that it should be Bayes consistent, which we will illustrate in the following section. 

\subsection{Bayes Consistent Loss Function}

A loss function $L$ is said to be \textit{Bayes consistent} \cite{tewari2007consistency} if the $L$-risk of $f$ gets close to the optimal when the risk of $f$ approaches the lower-bound of Bayes risk. That is, $f^* \rightarrow \sup_f R_L(f)$ implies $f^* \rightarrow \sup_f R_{0-1}(f)$.
According to the work of \citeauthor{bartlett2006convexity}, \citeyear{bartlett2006convexity}, a convex loss function $\Phi: R \rightarrow R^+$ is \textit{Bayes consistent} if it is differentiable at 0 and $\Phi^\prime(0) < 0$, and any minimizer $g$ of
\begin{equation*}
R_{\Phi}(g) = \mathbb{E}_{\rbm{X}, \rm{Y}_{\rbm{X}}} \Phi\left(g(\bm{x})(y_{\bm{x}}-\frac{1}{2})\right)
\end{equation*}
yields a Bayes consistent classifier, that is, $\text{P}(\rm{Y} =1|\rbm{X}=\bm{x}) > \frac{1}{2} \Rightarrow g(\bm{x}) > 0$ and $\text{P}(\rm{Y} = 1|\rbm{X} = \bm{x}) < \frac{1}{2} \Rightarrow g(\bm{x}) < 0$.

\section{Methodology}
In this work, we follow the assumption that instances are identical independent (i.i.d) and the instance distribution is independent to the bag label given its instance label.
According to this assumption, we have:
\begin{align}
    & \text{P}(\rm{Y}_{\rbm{X}}=0|\rm{Y}_{\rbm{B}}=0) = 1 \label{eq:5}\\
    & \text{P}(\rbm{X}|\rm{Y}_{\rbm{X}}, \rm{Y}_{\rbm{B}}) = \text{P}(\rbm{X}| \rm{Y}_{\rbm{X}}) \label{eq:6}\\
    & \text{P}(\rbm{X}|\rm{Y}_{\rbm{X}}=0) =\text{P}(\rbm{X}|\rm{Y}_{\rbm{B}}=0) \label{eq:7}\\
    & \text{P}(\rm{Y}_{\rbm{X}}=1)=\text{P}(\rm{Y}_{\rbm{X}}=1, \rm{Y}_{\rbm{B}}=1). \label{eq:8}
\end{align}
In the following, we show how to estimate the instance-level label prediction loss $R_L$ without using instance labels built on this assumption.

\subsection{Estimate $R_L$  without Instance Labels}
In this section, we prove that $R_L$ can be unbiasely estimated without using instance labels as detailed in the following theorem. 
\begin{theorem}{1} \label{theorem:unbias}
$R_L$ can be unbiasely estimated by:
\begin{equation} \label{eq:our_risk_estimate}
\begin{split}
&\hat{R}^u_L(f) = \frac{1}{\sum_{\bm{b}\in S} |\bm{b}|} \sum_{\bm{b}\in S} \sum_{\bm{x} \in \bm{b}} L(f(\bm{x}), 1) \\
& + \frac{\text{P}(\rm{Y}_{\rbm{X}}=0)}{\sum_{\bm{b}\in S^-} |\bm{b}|} \sum_{\bm{b}\in S^-} \sum_{\bm{x} \in \bm{b}} \left(L(f(\bm{x}), 0) - L(f(\bm{x}), 1)\right).
\end{split}
\end{equation}
Here the superscript $u$ of $\hat{R}^u_L$ means that the risk is estimated in the unsupervised setting, where instance labels are not available.
\end{theorem}

\begin{proof}
$R_L$ can be reformulated as 
\begin{equation*}
R_L(f) = \mathbb{E}_{\rbm{X}, \rm{Y}_{\rbm{B}}, \rm{Y}_{\rbm{X}}} L(f(\bm{x}), y_{\bm{x}}),
\end{equation*}
and divided into two parts by bag label:
\begin{equation} \label{eq:divide_by_yb}
\begin{split}
    &R_L(f)  = \text{P}(\rm{Y}_{\rbm{B}}=0) \mathbb{E}_{\rbm{X}|\rm{Y}_{\rbm{B}}=0} \mathbb{E}_{\rm{Y}_{\rbm{X}}|\rbm{X}, \rm{Y}_{\rbm{B}}=0}  L(f(\bm{x}), y_{\bm{x}}) \\
    &\quad + \text{P}(\rm{Y}_{\rbm{B}}=1) \mathbb{E}_{\rbm{X}|\rm{Y}_{\rbm{B}}=1} \mathbb{E}_{\rm{Y}_{\rbm{X}}|\rbm{X}, \rm{Y}_{\rbm{B}}=1}  L(f(\bm{x}), y_{\bm{x}}) \\
\end{split}
\end{equation}
According to Eq. (\ref{eq:5}), we have
\begin{equation} \label{eq:12}
\begin{split}
&\text{P}(\rm{Y}_{\rbm{B}}=0) \mathbb{E}_{\rbm{X}|\rm{Y}_{\rbm{B}}=0} \mathbb{E}_{\rm{Y}_{\rbm{X}}|\rbm{X}, \rm{Y}_{\rbm{B}}=0} \; L(f(\bm{x}), y_{\bm{x}}) \\ 
&= \text{P}(\rm{Y}_{\rbm{B}}=0) \mathbb{E}_{\rbm{X}|\rm{Y}_{\rbm{B}}=0} L(f(\bm{x}), 0).
\end{split}
\end{equation} 
The second term of the right item of Eq. (\ref{eq:divide_by_yb}) can be further formulated as:
\begin{align} \label{eq:divide_by_yx}
\begin{split}
    &\text{P}(\rm{Y}_{\rbm{B}}=1) \mathbb{E}_{\rbm{X}|\rm{Y}_{\rbm{B}}=1} \mathbb{E}_{\rm{Y}_{\rbm{X}}|\rbm{X}, \rm{Y}_{\rbm{B}}=1} L(f(\bm{x}), y_{\bm{x}}) \\
    & = \text{P}(\rm{Y}_{\rbm{B}}=1)\mathbb{E}_{\rm{Y}_{\rbm{X}}|\rm{Y}_{\rbm{B}}=1} \mathbb{E}_{\rbm{X}|\rm{Y}_{\rbm{X}}, \rm{Y}_{\rbm{B}}=1}  L(f(\bm{x}), y_{\bm{x}}) \\
    & =  \text{P}(\rm{Y}_{\rbm{X}}=1, \rm{Y}_{\rbm{B}}=1) \mathbb{E}_{\rbm{X}|\rm{Y}_{\rbm{X}}=1, \rm{Y}_{\rbm{B}}=1} L(f(\bm{x}), 1) \\
    & + \text{P}(\rm{Y}_{\rbm{X}}=0, \rm{Y}_{\rbm{B}}=1) \mathbb{E}_{\rbm{X}|\rm{Y}_{\rbm{X}}=0, \rm{Y}_{\rbm{B}}=1} L(f(\bm{x}), 0).
\end{split}
\end{align}
According to Eq. (\ref{eq:6}) and (\ref{eq:7}), we have
\begin{equation*}
\begin{split}
 \mathbb{E}_{\rbm{X}|\rm{Y}_{\rbm{X}}=1, \rm{Y}_{\rbm{B}}=1} L(f(\bm{x}), 1) &= \mathbb{E}_{\rbm{X}|\rm{Y}_{\rbm{X}}=1} L(f(\bm{x}), 1) \\
 \mathbb{E}_{\rbm{X}|\rm{Y}_{\rbm{X}}=0, \rm{Y}_{\rbm{B}}=1} L(f(\bm{x}), 0) &= \mathbb{E}_{\rbm{X}|\rm{Y}_{\rbm{B}}=0} L(f(\bm{x}), 0),
\end{split}
\end{equation*}
thus:
\begin{align} \label{eq:14}
\begin{split}
    &\text{P}(\rm{Y}_{\rbm{B}}=1) \mathbb{E}_{\rbm{X}|\rm{Y}_{\rbm{B}}=1} \mathbb{E}_{\rm{Y}_{\rbm{X}}|\rbm{X}, \rm{Y}_{\rbm{B}}=1} L(f(\bm{x}), y_{\bm{x}}) \\
    & =  \text{P}(\rm{Y}_{\rbm{X}}=1, \rm{Y}_{\rbm{B}}=1) \mathbb{E}_{\rbm{X}|\rm{Y}_{\rbm{X}}=1} L(f(\bm{x}), 1) \\
    &+ \text{P}(\rm{Y}_{\rbm{X}}=0, \rm{Y}_{\rbm{B}}=1) \mathbb{E}_{\rbm{X}|\rm{Y}_{\rbm{B}}=0} L(f(\bm{x}), 0).
\end{split}
\end{align}
Because 
\begin{equation*}
\resizebox{\columnwidth}{!}{$
    \text{P}(\rbm{X}) = \text{P}(\rm{Y}_{\rbm{X}}=1) \text{P}(\rbm{X}|\rm{Y}_{\rbm{X}}=1) + \text{P}(\rm{Y}_{\rbm{X}}=0) \text{P}(\rbm{X}|\rm{Y}_{\rbm{X}}=0),$}
\end{equation*}
we have
\begin{equation} \label{eq:15}
\begin{split}
    \mathbb{E}_{\rbm{X}|\rm{Y}_{\rbm{X}}=1} L(f(\bm{x}), 1) & = \frac{1}{\text{P}(\rm{Y}_{\rbm{X}}=1)}( \mathbb{E}_{\rbm{X}} L(f(\bm{x}), 1) \\
    & - \text{P}(\rm{Y}_{\rbm{X}}=0) \mathbb{E}_{\rbm{X}|\rm{Y}_{\rbm{B}}=0} L(f(\bm{x}), 1) ).
\end{split}
\end{equation}
Combining Eq. (\ref{eq:8}), (\ref{eq:divide_by_yb}), (\ref{eq:12}), (\ref{eq:14}), and (\ref{eq:15}), we can obtain
\begin{align} \label{eq:eq_15}
\begin{split}
    &R_L(f) = \mathbb{E}_{\rbm{X}} L(f(\bm{x}), 1)\\
    &+\text{P}(\rm{Y}_{\rbm{X}}=0)\mathbb{E}_{\rbm{X}|\rm{Y}_{\rbm{B}}=0} \left(L(f(\bm{x}), 0)- L(f(\bm{x}), 1)\right).
\end{split}
\end{align}
Note that $\hat{R}^u_L(f)$ is an unbiased estimation of the right term of Eq. \ref{eq:eq_15}. Therefore, $\hat{R}^u_L(f)$ is an unbiased estimation of $R_L(f)$, completing the proof.
\end{proof}

\subsection{Consistency with Bounded Bayes Consistent Loss Function}
In this section, we analyze the consistency between $\hat{R}_L^u$ and $\hat{R}_L^s$ (all proofs appear in \textbf{Appendix A}\footnote{\url{https://drive.google.com/file/d/1MhvgHYtQo_9F2QcN0sGGNsK0lm6vBsEc/view?usp=sharing}}).

Let denote $L_M$ the Lipschitz constant that $L_M > \frac{\partial L(w, y)}{\partial w}, \forall w \in R$, denote $C_0 = \max_y L(0, y)$, and denote $\mathcal{H}$ a Reproducing Kernel Hilbert Space (RKHS). For each given $R > 0$, we consider as hypothesis space $\mathcal{H}_R$, the ball of radius $R$ in $\mathcal{H}$. Let $N(\epsilon)$ be the covering number of $\mathcal{H}_R$ following Theorem C in \cite{cucker2002mathematical}, and denote $m$ the instance number in $S^-$. Then we have the following theorem. 
\begin{theorem}{2} \label{theorem:consistency}
For a Bayes consistent loss function $L$, if it is bounded by $[0, M]$, then for any $\epsilon > 0$, 
\begin{equation}
\begin{split}
&\rm{P}\{S \in \mathcal{D}|\sup_{f \in \mathcal{H}_R} |{R}_L-\hat{R}_L^u| \leq \epsilon\} \\
&\geq 1 - 2N(\frac{\epsilon}{4(1+2\pi_0)L_M}) e^{-\frac{m\epsilon^2}{8(1+2\pi_0)^2B^2}},
\end{split}         
\end{equation}
where $\pi_0=\rm{P}(\rm{Y}_{\rbm{X}}=0)$ and $B = L_M M + C_0$.
\end{theorem}

\paragraph{Remark 2.} Let us think about what if $L$ is unbounded, or more specifically, not upper bounded. For a given example $\bm{x} \in S^-$ and $\bm{x} \not \in S^+$, its corresponding risk within $\hat{R}_L^u$ is $V(\bm{x}) = L(f(\bm{x}), 0)-\frac{n}{n+m}L(f(\bm{x}), 1)$. Because $L$ is not upper bounded, to achieve smaller risk on $\bm{x}$, $f$ can heavily overfit $\bm{x}$ making $L(f(\bm{x}), 1) \rightarrow +\infty$ and accordingly $L(f(\bm{x}), 0) \rightarrow 0$. Thus, $V(\bm{x}) \rightarrow -\infty$. From this analysis, we can expect that, when using a unbounded loss function and a flexible classifier, $\hat{R}_L^u$ will dramatically decrease to a far below zero value. This is indeed observed in our experiments.

\subsection{Formulate Bag Constrains of MIL}
Note that the above estimation of $R_L$ does not take account the instantiation of each bag, and it may not satisfy the bag constrains of MIL specified in Eq. (\ref{eq:3}). In this section, we show how to combine $R_L$ estimation with constrains of MIL from the risk view.  

From the risk view, the constrains can be formulated as:
\begin{align*}
&\max_{\bm{x} \in \bm{b}} L( f(\bm{x}), 0) - L(f(\bm{x}), 1) < 0, \forall \bm{b} \in S^- \\
&\max_{\bm{x} \in \bm{b}} L( f(\bm{x}), 0) - L(f(\bm{x}), 1) > 0, \forall \bm{b} \in S^+.
\end{align*}
Therefore, the optimization problem of Eq. (\ref{eq:3}) can be approximated by:
\begin{equation}
\begin{split}
\quad & \quad \quad \quad \quad \quad \quad \quad \min_f \; \hat{R}^u_L \\
& s.t. \max_{\bm{x} \in \bm{b}} L( f(\bm{x}), 0) - L(f(\bm{x}), 1) < 0, \forall \bm{b} \in S^- \\
& \quad \; \; \max_{\bm{x} \in \bm{b}} L( f(\bm{x}), 0) - L(f(\bm{x}), 1) > 0, \forall \bm{b} \in S^+.
\end{split}
\end{equation}
The above formulation assume that the data is exactly separable by $f$. This is not guaranteed for many cases. Even if $f$ pretty complicate and exactly separates the training bags, it usually suffers from severely overfitting problem and is susceptible to outliers. To regularize the above algorithm, we reformulate the optimization problem as:
\begin{equation}
\begin{split}
& \quad \quad \quad \quad \min_f \; \hat{R}^u_L +C (\sum_{i=1}^{|S^-|} \xi_i + \sum_{j=1}^{|S^+|} \xi_j)\\
& s.t. \max_{\bm{x} \in \bm{b}_i} L( f(\bm{x}), 0) - L(f(\bm{x}), 1) < \xi_i, \forall \bm{b}_i \in S^-, \\
& \quad \; \; \max_{\bm{x} \in \bm{b}_j} L( f(\bm{x}), 0) - L(f(\bm{x}), 1) > -\xi_j, \forall \bm{b}_j \in S^+, \\
& \quad \quad \xi_i > 0,  i=1, \cdots, |S^-|; \quad \xi_j > 0, j=1, \cdots, |S^+|.
\end{split}
\end{equation}
where $|S^+|$ denotes the bag number in $S^+$, and $C \in R^+$ is a hyper-parameter to control the relative weighting between the twin goals of making $\hat{R}^u_L$ small and of ensuring that $f$ satisfies the bag constrains of MIL. In this work, we set $C=10$.

\section{Experiment}

\subsection{Choose Loss Function}

In this work, we consider the mean square error (MSE) loss function $L_{MSE}$ and the cross entropy (CE) loss function $L_{CE}$ to perform the task. Denote $w=\sigma(f(\bm{x}))$ where $\sigma$ is the sigmoid function, we can check that $L_{MSE}(f(\bm{x}), y_{\bm{x}}) = (w-y_{\bm{x}})^2=(\frac{1}{2}-2(w-\frac{1}{2})(y_{\bm{x}}-\frac{1}{2}))^2$ and its corresponding $\Phi=(\frac{1}{2}-x)^2$ when $g(\bm{x})=2(\sigma(f(\bm{x})-\frac{1}{2})$. Therefore, $L_{MSE}$ is Bayes consistent. We can also check that $L_{CE}(f(\bm{x}), y_{\bm{x}}) = -y\log w-(1-y)\log(1-w)=\log(1+e^{-2(y-\frac{1}{2})f(y_{\bm{x}})})$ and its corresponding $\Phi(x)=\log (1+e^{-x})$ when $g(\bm{x})=2f(\bm{x})$. Therefore, $L_{CE}$ is also a Bayes consistent loss function. In addition, it is easy to know that the value range of $L_{MSE}$ is $(0, 1)$ and that of $L_{CE}$ is $(0, + \infty)$. Therefore, $L_{MSE}$ is bounded while $L_{CE}$ is not. In this work, we apply $L_{MSE}$ as the loss function if without further illustration.

\subsection{Datasets}
To perform the empirical study, we designed experiments on four commonly used datasets, that is, MNIST \cite{lecun1998gradient}, SVHN \cite{netzer2011reading}, CIFAR10 \cite{krizhevsky2009learning}, and 20Newsgroup \cite{lang1995newsweeder}. A half classes of each dataset was mapped to the positive class and the other half ones were mapped to the negative class. For example, digits of 0-4 of MNIST were mapped to the positive class and those of 5-9 were mapped to the negative class. See \textbf{Appendix $\textbf{B}^1$} for label mapping detail of each dataset. For each dataset, we generated a training bag set and a testing bag set, respectively. Each bag set contains 3,000 positive bags and 3,000 negative bags. Let $D^+$ and $D^-$ denotes the positive and negative example set, respectively. Let $U(1, 9)$ denote a uniform distribution over integers ranging from 1 to 9. Each bag was generated in the following procedure: for each negative bag, we first sample the bag size $n$ from $U(1, 9)$, and then we randomly sample with replace $n$ negative examples from $D^-$; for each positive bag, we first sample the bag size $n$ from $U(1, 9)$ and the positive example number $m$ from $U(1, n)$, then we randomly sample $m$ with replace positive examples from $D^+$ and $n-m$ negative examples from $D^-$. Note that, to avoid overlap between training and testing bag sets, we solely sampled examples from the training data when constructing the training bag set and only sampled examples from the testing data when constructing the testing bag set. 

\begin{figure}[t!]
\centering
\includegraphics[width=\columnwidth]{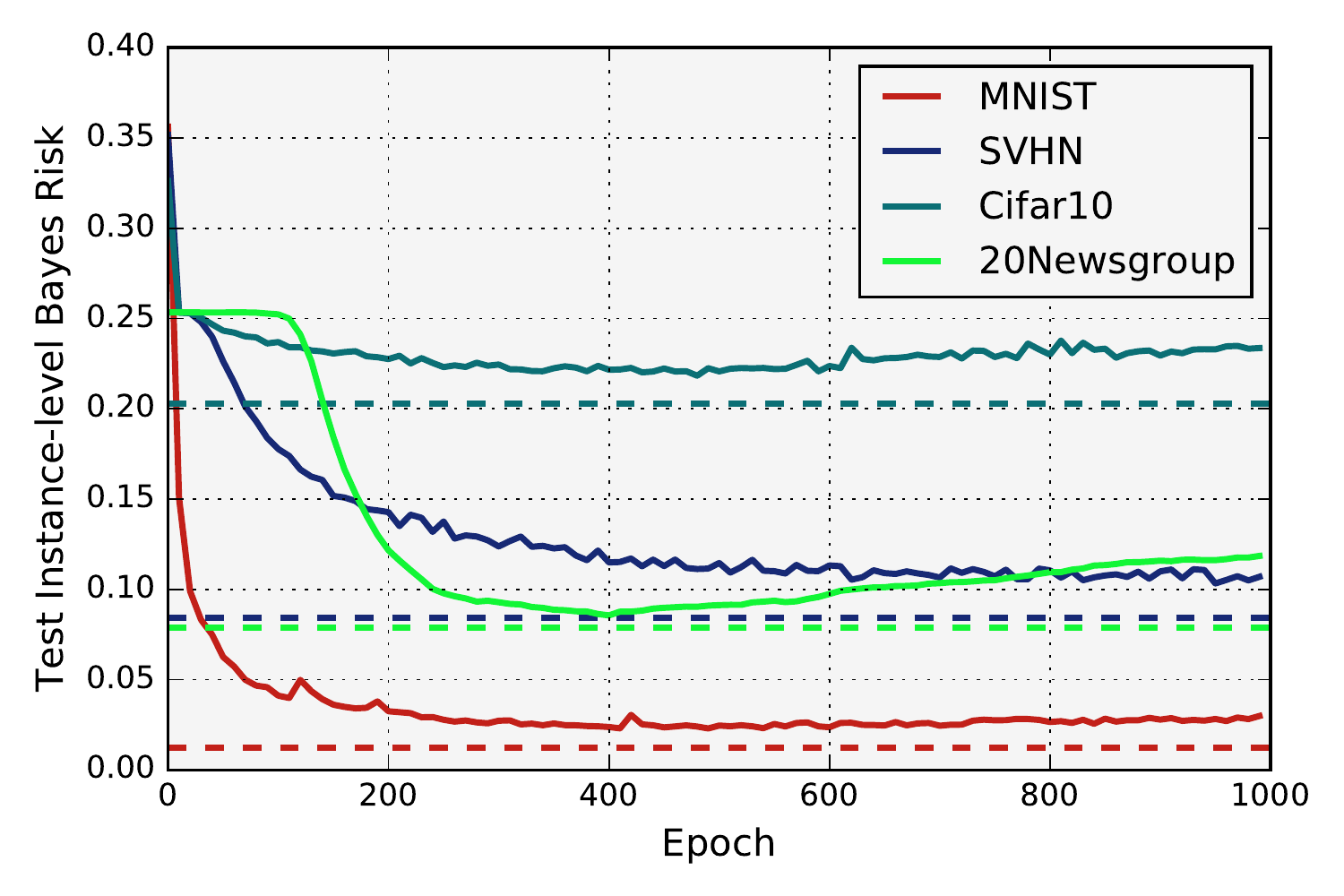}
\caption{Test Bayes risk of IMIL for the instance-level label prediction over epochs. The dot line of each dataset denotes the minimum Bayes risk achieved by the Sup model.}
\label{fig:ins_bayes}
\end{figure}

\begin{figure}[t!]
\centering
\includegraphics[width=\columnwidth]{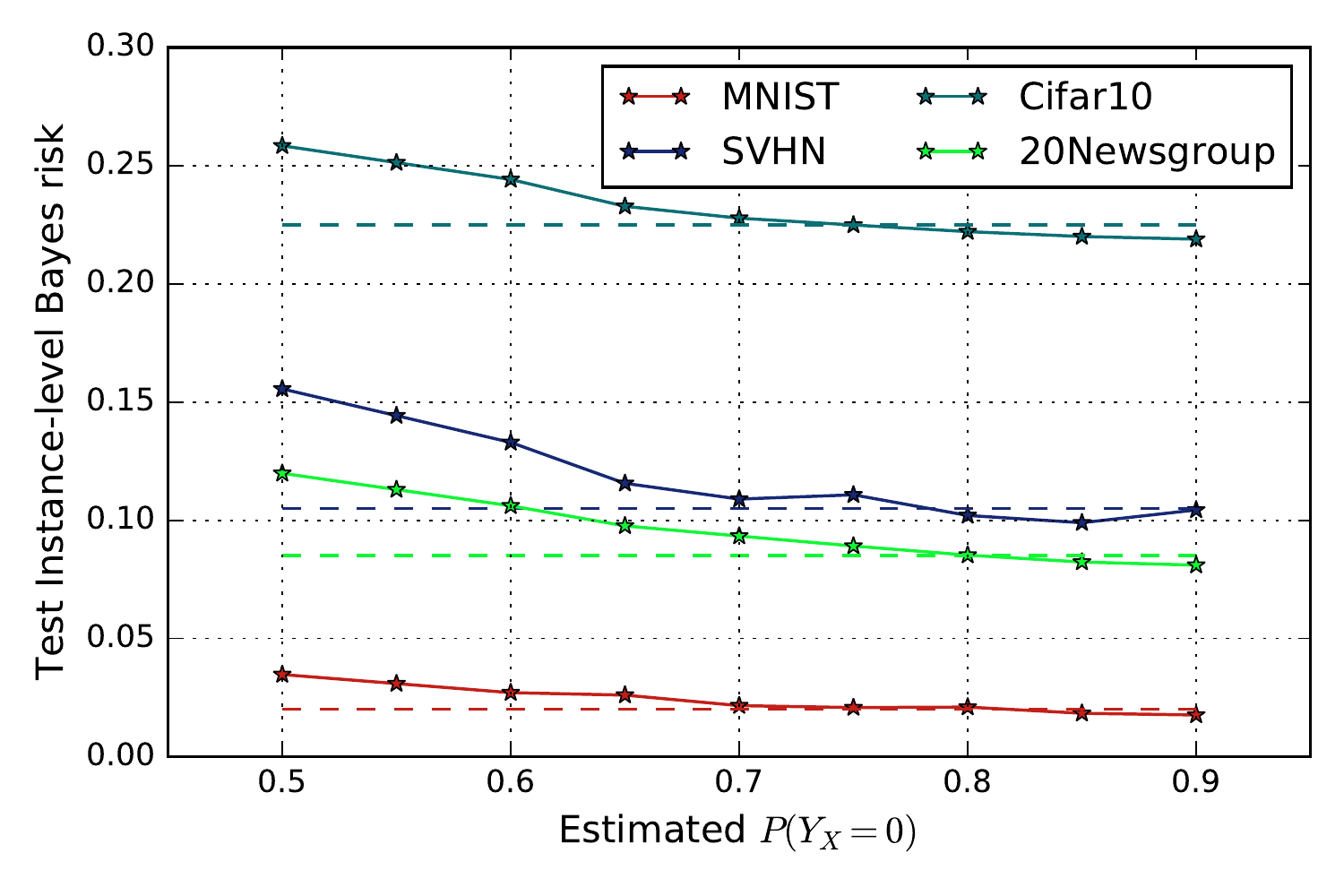}
\caption{Instance-level test Bayes risk of IMIL when using different values to estimate the true value of $\text{P}(\rm{Y}_{\rbm{X}}=0)$ (0.757). The dot line in the same color denotes the best performance of IMIL using the true value of $\text{P}(\rm{Y}_{\rbm{X}}=0)$.}
\label{fig:pi_study}
\end{figure}

\subsection{Compared Methods}
To evaluate the performance of our proposed algorithm, called both bag- and instance-level multiple instance learning (\textbf{BIMIL}), we introduced one of its variants, referred to \textbf{IMIL}, for a comparision. It only performs instance-level risk minimization (minimize $\hat{R}^u_L$) and does not consider the bag constrains of MIL. In addition, we introduced some typical MIL methods proposed in previously published works. They included three instance-level approaches, i.e., the diversity density (\textbf{DD}) \cite{maron1998framework} with one subject, multiple instance logistic regression (\textbf{MILR}) \cite{ray2005supervised}, and miNet \cite{wang2018revisiting}, and two bag-level approaches, i.e., \textbf{miFV} \cite{wei2017scalable}, and the neural-network based multiple instance learning with gated attention (\textbf{MIGA}) \cite{ilse2018attention}. For miFV, we extended an image into a pixel vector as the instance representation and we used the TF-IDF vector as the instance representation for the 20Newsgroup data. Finally, we compared it with the fully supervised model (\textbf{Sup}) trained on the instance-level labels.

\subsection{Implementation Detail}
For the image datasets, i.e., MNIST, SVHN, and Cifar10, we used two convolutional layers following two fully-connected layers to implement the classifier $f$. All of these layers except the last one applied the relu activation functions. 
For 20Newsgroup, each instance was represented by the top 4,000 TF-IDF features and the classifier was implemented using two fully-connected layer with Softplus activation functions. Model parameters were all randomly initialized. 
Parameter updating was implemented using the RMSprop \cite{graves2013generating} optimizer with learning rate set to be 1e-4. Within the experiments, we assumed that the value of $\text{P}(\rm{Y}_{\rbm{X}}=0)$ is knowns if without further illustration. See \textbf{Appendix B} for detail information for model implementation. 

\subsection{Instance-level Results}

\begin{figure*}[t!]
\begin{center}
\subfloat[SVHN]      {
\includegraphics[width=.47\linewidth, height=4.5cm]{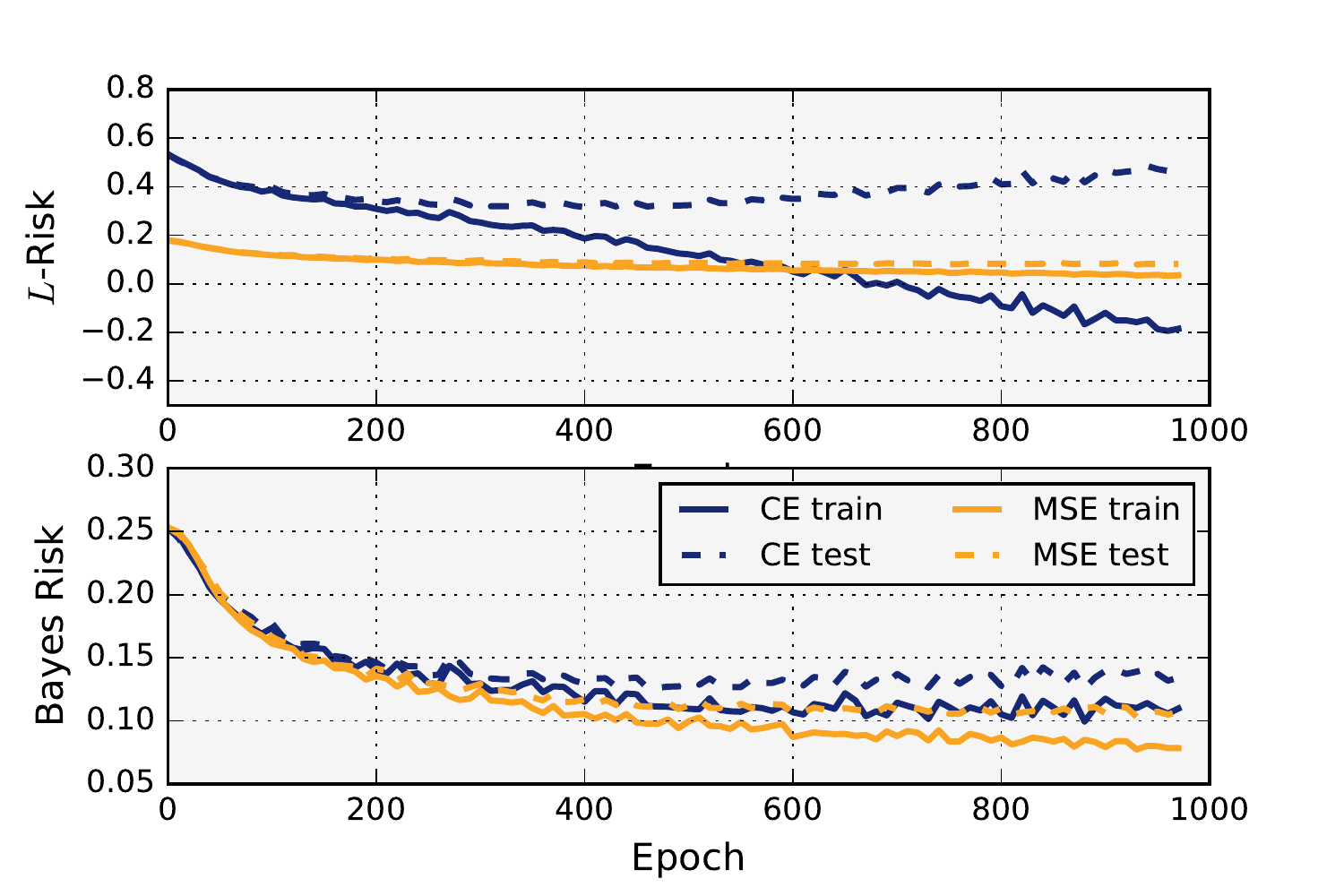}
}\quad
\subfloat[20Newsgroup]      {
\includegraphics[width=.47\linewidth, height=4.5cm]{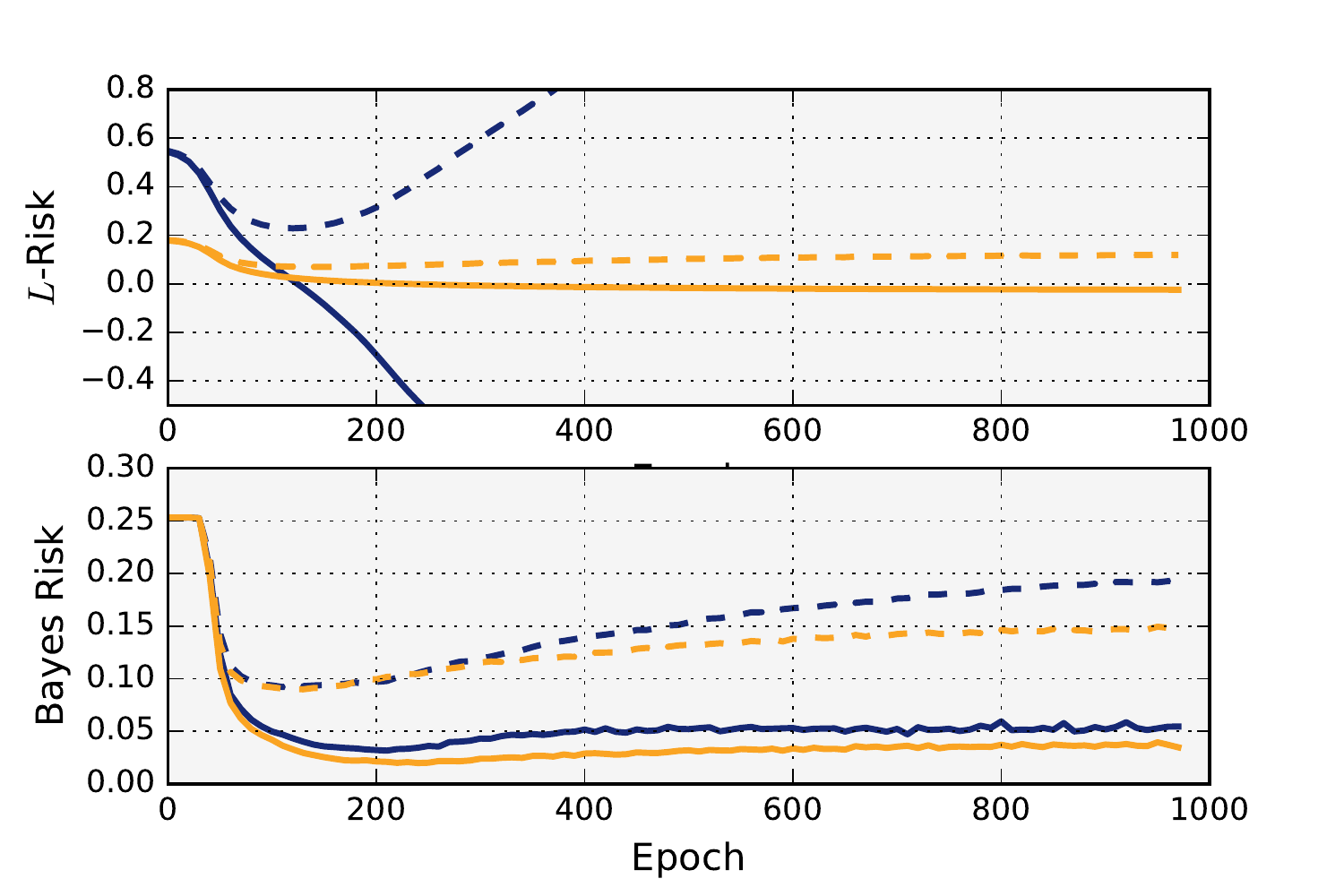}
}
\end{center}
\caption{Train-Test risk using different loss functions $L$ over epochs, on (a) SVHN Dataset and (b) 20Newsgroup Dataset. These figures share the same legend depicted in the bottom figure of (a).} 
\label{fig:loss_function}
\end{figure*}

\begin{table}[t]
\centering
\begin{tabular}{lcccc}
\toprule
Model  	& MNIST		& SVHN 	& Cifar10 	&20Newsgroup\\  \midrule	
Sup 		& 1.2	   	& 8.4	& 20.4		& 7.9\\ 
DD [1]	 		& 1.9  		& 12.8	& 22.2		& 8.9\\	
MILR [2] 	 	& 2.0	   	& 12.5	& 22.9		& 9.2\\
miNet [3]	 	& 2.1	   	& 12.3	& 22.7		& 9.5\\ \midrule
BIMIL  		& 1.7 	   	& 10.2	& 21.3		& 8.0\\
\bottomrule
\end{tabular}
\caption{Instance-level test Bayes risk ($\times 100$). The performance bottomline is 25.4 by predicting all instances as negative. [1]\protect \cite{maron1998framework}, [2]\protect\cite{ray2005supervised}, [3]\protect \cite{wang2018revisiting}.}
\label{table:instance_level_result}
\end{table}

We first studied the effectiveness of our proposed risk estimation method, i.e., IMIL. To perform this study, we compared its Bayes risks on test data with that of Sup. The experiment results are depicted in Figure \ref{fig:ins_bayes}. We can see that the minimum instance-level Bayes risks achieved by IMIL were all close to those of Sup on all of the tested four datasets.
We further compared the instance-level performance of our proposed algorithm, BIMIL, with two instance-level baselines, i.e., DD and MILR. Table \ref{table:instance_level_result} shows the results of these models on the four tested datasets. We can observe that BIMIL outperformed both DD and MILR for label prediction at the instance level and achieved compared results with the fully supervised model Sup. These observations verify the effectiveness of our proposed instance-level risk estimation method under the assumption of this work. 

We then studied the influence of loss function. We mainly studied what if the loss function is unbounded. To this end, we compared the performance of IMIL when using the mean square loss $L_{MSE}$ and the cross entropy loss $L_{CE}$, respectively, to instantiate $L$. Figure \ref{fig:loss_function} depicts the tendency of $L-risk$ and the Bayes risk by epoch when using these two loss functions on SVHN and 20Newsgroup datasets. From this figure, we can see that when using the cross entropy loss function to instantiate $L$, the $L$-risk on training data dropped quickly after a few epochs while that on testing data incrementally increased. This indicates that the model severely overfits training data. In the contrast, if we used the means square loss function to instantiate $L$, the overfitting was significantly reduced and achieved better instance-level classification performance on testing data. This is consistent with our analysis about the necessarity of boundness of loss function in the "Consistency with Bounded Bayes Consistent Loss Function" section.

Finally, we studied the influence of using biased value of $\text{P}(\rm{Y}_{\rbm{X}}=0)$ to the proposed algorithm. 
According to our data generation process, we have $\text{P}(\rm{Y}_{\rbm{X}}=0) \in [0.5, 0.9]$. The minimum value 0.5 occurs when every positive bag comprises of positive instances only, and the maximum value 0.9 occurs when every positive bag contains only one positive instance. Therefore, we tried the value ranging from 0.5 to 0.9 by 0.05. Figure \ref{fig:pi_study} depicts the performance of IMIL using different values of $\text{P}(\rm{Y}_{\rbm{X}}=0)$. We can first observe from the figure that IMIL performs quite robust around the true value of $\text{P}(\rm{Y}_{\rbm{X}}=0)$. In addition, for some over-estimated values, it even performs better than using the true value. 
Finally, we can see that it usually performs better for IMIL with over-estimated values than with under-estimated values. 

\subsection{Bag-level Results}
To evaluate the bag-level performance of the proposed algorithm, we compared it with the five published baselines and its variant IMIL. Table \ref{table:bag_level_results} shows the performance of these models. We can first observe that the proposed algorithm, BIMIL, achieved comparable results with state-of-the-art MIL approaches for the bag-level label prediction. We can also observe that IMIL, 
which does not incorporate with constrains of MIL, also achieve acceptable results compared with existent models. This shows the applicability of our proposed algorithm for the bag-level label prediction. Of course, by comparing the results of IMIL and BIMIL, we know that incorporating with constrains of MIL can bring further improvement to the bag-level label prediction. Overall, the above experiments proved the effectiveness of our proposed algorithm for both the bag- and instance- level label prediction in MIL. 

\begin{table}[t]
\centering
\resizebox{\columnwidth}{!}{
\begin{tabular}{lcccc}
\toprule
Model	& MNIST & SVHN 	& Cifar10 	& 20Newsgroup  \\  \hline 
Sup		& 2.8	& 14.7	& 28.7		& 10.8 \\ \midrule
DD [1]		& \textbf{2.7}	& \textbf{13.6}	& 31.7		& 12.6\\
MILR [2] 	& 3.6	& 15.7	& 33.1		& 14.3\\
miNet [3] 	 	& 3.7	& 15.3	& 33.7		& 13.5\\ 
miFV [4]	& 3.8	& 15.8	& 34.2		& 15.0\\
MIGA [5] 	& 3.1	& 15.1	& 32.4		& \textbf{11.4}\\ \midrule
IMIL 	& 2.9	& 15.7	& 31.8		& 13.3\\
BIMIL 	& \textbf{2.7}	& \textbf{13.6}	& \textbf{31.0}		& 13.0\\
\bottomrule
\end{tabular}}
\caption{Bag-level test Bayes risk ($\times 100$). Bold values indicate the best performance within methods not using instance labels. [1]\protect \cite{maron1998framework}, [2]\protect\cite{ray2005supervised}, [3]\protect \cite{wang2018revisiting}, [4]\protect\cite{wei2017scalable}, [5]\protect\cite{ilse2018attention}.}\label{table:bag_level_results}
\end{table}

\section{Conclusion}
This work proposes a novel multiple instance learning algorithm to address the instance-level label prediction problem in MIL. Different from existent MIL approaches, the loss function of our proposed algorithm was specifically defined on the instance-level label prediction without using instance labels. We theoretically prove the effectiveness of this algorithm by using a bounded Bayes consistent loss function, under the i.i.d assumption. Experimental studies on both image and text datasets show that the proposed algorithm can achieve comparative performance for both the bag- and instance- level label prediction. 

\bibliographystyle{ijcai19}
\fontsize{10pt}{10.6pt} \selectfont
\bibliography{MIL}

\appendix

\section{Proof of Theorem 1}
Let denote $L_M$ the Lipschitz constant that $L_M > \frac{\partial L(w, y)}{\partial w}, \forall w \in R$, denote $C_0 = \max_y L(0, y)$, and denote $\mathcal{H}$ a Reproducing Kernel Hilbert Space. For each given $R > 0$, we consider as hypothesis space $\mathcal{H}_R$, the ball of radius $R$ in the RKHS $\mathcal{H}$. Let $N(\epsilon)$ be the covering number of $\mathcal{H}_R$ following Theorem C in \cite{cucker2002mathematical}. Let denote $n$ the instance number in $S^+$ and $m$ that in $S^-$. Then we have the following theorem. 
\begin{theorem}{1}
For a Bayes consistent loss function $L$, if it is bounded by $[0, M]$, then for any $\epsilon > 0$, 
\begin{equation}
\begin{split}
&\rm{P}\{S \in \mathcal{D}|\sup_{f \in \mathcal{H}_R} |{R}_L-\hat{R}_L^u| \leq \epsilon\} \\
&\geq 1 - 2N(\frac{\epsilon}{4(1+2\pi_0)L_M}) e^{-\frac{m\epsilon^2}{8(1+2\pi_0)^2B^2}},
\end{split}         
\end{equation}
where $\pi_0=\rm{P}(\rm{Y}_{\rbm{X}}=0)$ and $B = L_M M + C_0$.
\end{theorem}

\begin{proof}
Let denote $\hat{R}_{L}(f)$ the empirical estimation of $R_L(f)$ with $k$ randomly labeled examples. Since $L$ is bounded, $C_0$, $M$, and $B$ are finite. According to the Lemma in \cite{rosasco2004loss} we have:
\begin{equation} \label{eq:lemma}
\begin{split}
        &\rm{P}\{S \in \mathcal{D}|\sup_{f \in \mathcal{H}_R} |{R}_L(f)-\hat{R}_L(f)| \leq \epsilon\} \\
        &\geq 1 - 2N(\frac{\epsilon}{4L_M}) e^{-\frac{k\epsilon^2}{8B^2}},
\end{split}   
\end{equation}
The empirical estimation error of ${R}_L(f) - \hat{R}^u_L(f)$ can be written as:
\begin{equation}
\begin{split}
&R_L(f) - \hat{R}^u_L(f) \\
&= \left(\mathbb{E}_{\rbm{X}}L(f(\bm{x}, 1)-\frac{1}{\sum_{\bm{b}\in S} |\bm{b}|} \sum_{\bm{b}\in S} \sum_{\bm{x} \in \bm{b}} L(f(\bm{x}), 1)\right) \\
& + \pi_0\left(\mathbb{E}_{\rbm{X}|\rbm{B}=0}L(f(\bm{x}, 0)-\frac{1}{\sum_{\bm{b}\in S^-} |\bm{b}|} \sum_{\bm{b}\in S^-} \sum_{\bm{x} \in \bm{b}} L(f(\bm{x}), 0)\right)\\
& + \pi_0\left(\mathbb{E}_{\rbm{X}|\rbm{B}=0}L(f(\bm{x}, 1)-\frac{1}{\sum_{\bm{b}\in S^-} |\bm{b}|} \sum_{\bm{b}\in S^-} \sum_{\bm{x} \in \bm{b}} L(f(\bm{x}), 1)\right)
\end{split}
\end{equation}
Thus,
\begin{equation}
\begin{split}
&|R_L(f) - \hat{R}^u_L(f)| \\
&\leq \left|\mathbb{E}_{\rbm{X}}L(f(\bm{x}), 1)-\frac{1}{\sum_{\bm{b}\in S} |\bm{b}|} \sum_{\bm{b}\in S} \sum_{\bm{x} \in \bm{b}} L(f(\bm{x}), 1)\right| \\
& + \pi_0\left|\mathbb{E}_{\rbm{X}|\rbm{B}=0}L(f(\bm{x}), 0)-\frac{1}{\sum_{\bm{b}\in S^-} |\bm{b}|} \sum_{\bm{b}\in S^-} \sum_{\bm{x} \in \bm{b}} L(f(\bm{x}), 0)\right|\\
& + \pi_0\left|\mathbb{E}_{\rbm{X}|\rbm{B}=0}L(f(\bm{x}), 1)-\frac{1}{\sum_{\bm{b}\in S^-} |\bm{b}|} \sum_{\bm{b}\in S^-} \sum_{\bm{x} \in \bm{b}} L(f(\bm{x}), 1)\right|
\end{split}
\end{equation}
Let $\rm{I}_L(\rbm{X})$ denote $$\mathbb{E}_{\rbm{X}}L(f(\bm{x}, 1)-\frac{1}{\sum_{\bm{b}\in S} |\bm{b}|} \sum_{\bm{b}\in S} \sum_{\bm{x} \in \bm{b}} L(f(\bm{x}), 1).$$ According to Eq. \ref{eq:lemma}, we have:
\begin{equation}
\begin{split}
&\rm{P}\{S \in \mathcal{D}|\sup_{f \in \mathcal{H}_R} |\rm{I}_L(\rbm{X})| \leq \epsilon\} \\
& \geq 1-2N(\frac{\epsilon}{4L_M}) e^{-\frac{(m+n)\epsilon^2}{8B^2}}
\end{split}
\end{equation}
Similarly, let $\rm{I}_L(\rbm{X}|\rbm{B}=0, 0)$ denote $$\mathbb{E}_{\rbm{X}|\rbm{B}=0}L(f(\bm{x}, 0)-\frac{1}{\sum_{\bm{b}\in S^-} |\bm{b}|} \sum_{\bm{b}\in S^-} \sum_{\bm{x} \in \bm{b}} L(f(\bm{x}), 0),$$ and $\rm{I}_L(\rbm{X}|\rbm{B}=0, 10)$ denote $$\mathbb{E}_{\rbm{X}|\rbm{B}=0}L(f(\bm{x}, 1)-\frac{1}{\sum_{\bm{b}\in S^-} |\bm{b}|} \sum_{\bm{b}\in S^-} \sum_{\bm{x} \in \bm{b}} L(f(\bm{x}), 1),$$
we have:
\begin{equation}
\begin{split}
&\rm{P}\{S \in \mathcal{D}|\sup_{f \in \mathcal{H}_R} |\rm{I}_L(\rbm{X}|\rbm{B}=0, 0)| \leq \epsilon\} \\
& \geq 1-2N(\frac{\epsilon}{4L_M}) e^{-\frac{m\epsilon^2}{8B^2}},
\end{split}
\end{equation}
and 
\begin{equation}
\begin{split}
&\rm{P}\{S \in \mathcal{D}|\sup_{f \in \mathcal{H}_R} |\rm{I}_L(\rbm{X}|\rbm{B}=0, 1)| \leq \epsilon\} \\
& \geq 1-2N(\frac{\epsilon}{4L_M}) e^{-\frac{m\epsilon^2}{8B^2}},
\end{split}
\end{equation}
Therefore, 
\begin{equation}
\begin{split}
&\rm{P}\{S \in \mathcal{D}|\sup_{f \in \mathcal{H}_R} |R_L(f)-\hat{R}_L^u(f)| \leq (1+2\pi_0)\epsilon \} \\
& \geq \min(1-2N(\frac{\epsilon}{4L_M}) e^{-\frac{m\epsilon^2}{8B^2}}, 1-2N(\frac{\epsilon}{4L_M}) e^{-\frac{(m+n)\epsilon^2}{8B^2}} ) \\
&= 1-2N(\frac{\epsilon}{4L_M}) e^{-\frac{m\epsilon^2}{8B^2}}
\end{split}
\end{equation}
The theorem follows replacing $\epsilon$ with $\frac{1}{1+2\pi_0}\epsilon$.
\end{proof}

\section{Implementation Detail}
\begin{table*}[]
\centering
\begin{tabular}{lcc}
\toprule
Dataset  	& Positive			& Negative  	\\  \midrule	
MNIST 	 	& 0-4               & 5-9 	\\
SVHN  	 	& 0-4			    & 5-9	\\  \midrule
Cifar10  	& \makecell[c]{"airplane", "automobile"\\ "bird", "cat", "deer"}			& \makecell[c]{"dog", "frog"\\ "horse", "ship", "truck"}\\  \midrule
20Newsgroup & \makecell[c]{"alt.atheism", "comp.graphics", "sci.crypt"\\ "comp.os.ms-windows.misc", "comp.windows.x"\\
"comp.sys.ibm.pc.hardware", "sci.space" \\"comp.sys.mac.hardware",  "sci.electronics", "sci.med"} & \makecell[c]{ "misc.forsale",
 "rec.autos"\\
 "rec.motorcycles",
 "rec.sport.baseball"\\
 "rec.sport.hockey", 
  "soc.religion.christian"\\
 "talk.politics.guns",
 "talk.politics.mideast"\\
 "talk.politics.misc",
 "talk.religion.misc"} 	\\ 
\bottomrule
\end{tabular}
\caption{Class mapping for each dataset.}
\label{table:label_mapping}
\end{table*}

Table \ref{table:label_mapping} shows classes mapped to the positive class and the negative class respectively in our experiments for multiple instance learning.

Table \ref{table:image_net} and \ref{table:text_net} show the model architectures used in our experiments on the three image datasets and the 20Newsgroup, respectively. 

\begin{table}[t]
\centering
\resizebox{\columnwidth}{!}{\begin{tabular}{cccccccc}
\toprule
&Layer	&Operation	&channels	&width	&height	\\
&0		&Input		&1	&28	&28	 \\
&1		&Conv (relu)		&32	&28	&28	\\
&2		&Max pool	&32	&14	&14	\\
&3		&Conv (relu)		&32	&14	&14	\\
&4	&Max pool	&32	&7	&7	\\
&6	&Dense	(relu) &512\\
&7	&Dropout	&512	\\
&8	&Dense	(sigmoid)		&1	\\
 \bottomrule
\end{tabular}}
\caption{Model architecture for image datasets.}
\label{table:image_net}
\end{table}

\begin{table}[t]
\centering
{\begin{tabular}{cccccccc}
\toprule
&Layer	&Operation			&dimension\\
&0		&Input				&4,096	\\
&1		&Dense (softplus)	&256\\
&2		&Dense (softplus)	&64	\\
&3		&Dense	(sigmoid)	&1	\\
 \bottomrule
\end{tabular}}
\caption{Model architecture for 20Newsgroup.}
\label{table:text_net}
\end{table}

\end{document}